\documentclass{article}
\usepackage[table,xcdraw]{xcolor}
\usepackage[preprint]{tmlr}
\usepackage{amsmath}
\usepackage{amssymb}
\usepackage{mathtools}
\usepackage{theorem}
\usepackage{csquotes}
\usepackage{tikz}
\usepackage{tikz-3dplot}
\usepackage{hyperref}
\usepackage{multirow}

\newtheorem{definition}{Definition}
\newtheorem{theorem}{Theorem}
\newtheorem{corollary}[theorem]{Corollary}
\newtheorem{hypothesis}[theorem]{Hypothesis}
\newtheorem{proof}{Proof}

\newcommand{\reffig}[1]{Figure~\ref{#1}}
\newcommand{\reftable}[1]{Table~\ref{#1}}

\newcommand{\mdot}[2]{\left\langle #1, #2 \right\rangle}
\newcommand{\norm}[1]{\left\Vert #1 \right\Vert}
\newcommand{\sqnorm}[1]{\norm{#1}^2}
\newcommand{\reals}{\mathbb{R}}
\newcommand{\nsphere}[1]{\mathcal{S}_{#1}}
\newcommand{\blank}{\mathord{{}\_{}}}

\newcommand{\reftheorem}[1]{Theorem~\ref{#1}}
\newcommand{\refcorollary}[1]{Corollary~\ref{#1}}

\newcommand{\veccat}[2]{\left[ #1, #2\right]}
\newcommand{\simplified}[1]{\widetilde{#1}}

\title{Explicit Formulae to Interchangeably use Hyperplanes\newline
and Hyperballs using Inversive Geometry}
\author{\name Erik Thordsen \email erik.thordsen@tu-dortmund.de \\
      \addr Fakultät für Informatik, Data Mining\\
      TU Dortmund University
      \AND
      \name Erich Schubert \email erich.schubert@tu-dortmund.de \\
      \addr Fakultät für Informatik, Data Mining\\
      TU Dortmund University
}

\def\month{05}  
\def\year{2024} 

\begin{document}

\maketitle

\begin{abstract}
	Many algorithms require discriminative boundaries, such as separating hyperplanes or hyperballs, or are specifically designed to work on spherical data.
	By applying inversive geometry, we show that the two discriminative boundaries can be used interchangeably, and that general Euclidean data can be transformed into spherical data, whenever a change in point distances is acceptable.
	We provide explicit formulae to embed general Euclidean data into spherical data and to unembed it back.
	We further show a duality between hyperspherical caps, i.e., the volume created by a separating hyperplane on spherical data, and hyperballs and provide explicit formulae to map between the two.
	We further provide equations to translate inner products and Euclidean distances between the two spaces, to avoid explicit embedding and unembedding.
	We also provide a method to enforce projections of the general Euclidean space onto hemi-hyperspheres and propose an intrinsic dimensionality based method to obtain \enquote{all-purpose} parameters.
	To show the usefulness of the cap-ball-duality, we discuss example applications in machine learning and vector similarity search.
\end{abstract}

\section{Introduction}
Some practical applications in machine learning are specifically designed to work on spherical or non-spherical data, such as HIOB by \cite{DBLP:conf/sisap/ThordsenS23}.
Extending these algorithms to non-spherical data can require a lot of effort.
In other cases, algorithms are specifically designed to work with separating hyperplanes, when one would perhaps desire to use hyperballs instead, such as SVMs which have been extended to use balls for the Support Vector Data Description by \cite{DBLP:journals/ml/TaxD04}.
Again, redesigning the algorithms can be quite difficult if not nigh impossible.
In this paper, we will provide concise formulae to eliminate both of these issues -- at least if a locally-linear change in point distances is acceptable.
We propose the use of inversive geometry to transfer general Euclidean datasets into spherical data and provide the necessary formulae to jump between hyperballs in the general Euclidean and hyperplanes in the spherical space.
Since our motivation is primarily to make Euclidean spaces spherical (in order to be able to use machine learning algorithms designed for spherical data), we will call the general Euclidean space the \enquote{original space} and the spherical space the \enquote{embedding space}.
While most of the proposed ideas have been well known for low-dimensional spaces in the field of inversive geometry for a long time, dating back to the early 19th century \citep[p.{} 279]{Coolidge1947AHO}, to our knowledge, no one has yet provided explicit formulae for the embedding and unembedding functions for arbitrary-dimensional spaces.
Neither did we find the necessary formulae to map hyperballs to hyperplanes and vice versa and to translate inner products and Euclidean distances between the two spaces.
Most literature on inversive geometry is limited to three dimensions, due to its usefulness in visualization techniques such as geographic maps. While it is often easy to show that some surface type is mapped to another surface type, giving the exact values for ball centers and radii requires quite exhausting algebraic transformations.
Today's machine learning methods, however, often involve high-dimensional embeddings, and much of the computational improvements are focused on optimizing the computation of inner products.
This paper provides equations useful to the practically inclined machine learning researcher, who can now with ease explore the possibilities of using spherical data in their algorithms.

We provide formulae to \enquote{move} between general Euclidean and spherical spaces of one additional dimension using arbitarily-dimensional stereographic projections.
The (inverse) stereographic projection used for that process is visualized in \reffig{fig:stereographic_demo}, where the poles are plotted as arrows.
We show that hyperspherical caps given specific poles get mapped to hyperballs in the general Euclidean space, for which we provide explicit formulae for center and radius.
We also provide the inverse map and equations to translate inner products and Euclidean distances between the two spaces.
We further show, that the result can be extended to arbitrary poles, resulting in a map between hyperspherical caps and a specific subclass of hyperellipsoids.
We provide parameters to enforce projections of the general Euclidean space onto hemi-hyperspheres and propose an intrinsic-dimensionality-based method to obtain \enquote{all-purpose} parameters.
To showcase the usefulness of our results, we include examples of how to use the provided methods in practical settings.

\begin{figure}
	\centering
	\begin{tikzpicture}[inner sep=0pt]
		\node[anchor=west,inner sep=0pt] at (0,0) {\includegraphics[scale=.75]{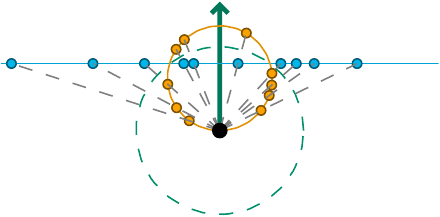}};
		\node[anchor=west,inner sep=0pt] at (7,0) {\includegraphics[scale=.5,trim=70 0 70 0,clip]{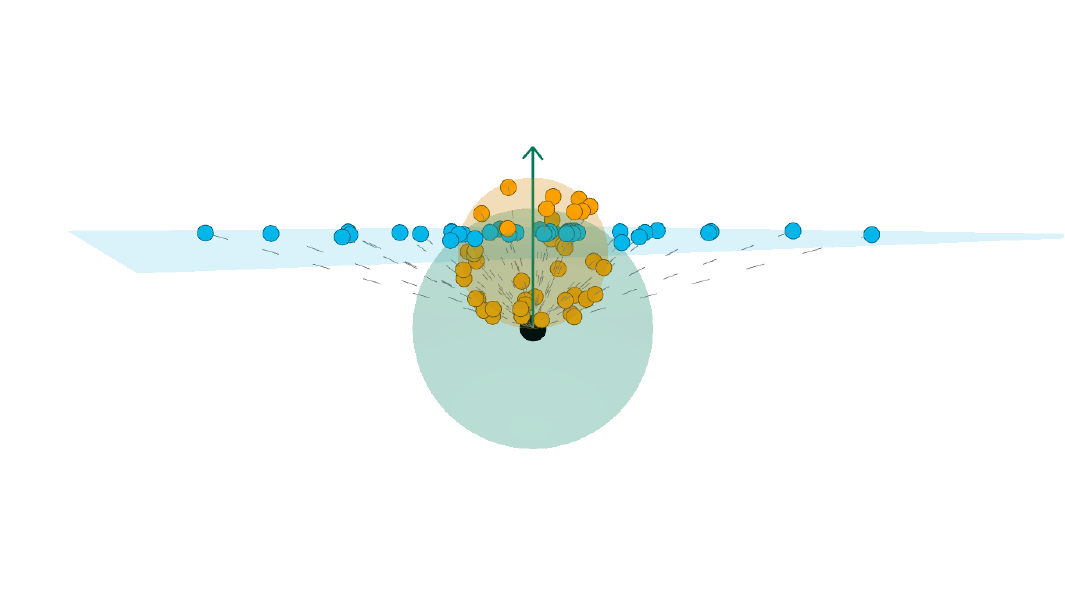}};
	\end{tikzpicture}
	\caption{
		Visualization of the (inverse) stereographic projection from one- and two-dimensional data to two- and three-dimensional spheres.
		By offsetting the data to a plane in an additional dimension (blue) and inverting the vector norms, i.e., mirroring each point with the unit sphere (green) around the origin (black), we obtain a spherical distribution (orange).
	}
	\label{fig:stereographic_demo}
\end{figure}

\section{Hyperspherical Embeddings}

Assuming an arbitrarily indexed dataset $X \subset \reals^d$, a unit-length inversion direction vector $v \in \reals^{d+1}$, and a scale factor $s \in \reals_{>0}$.
We define the affine hyperplane embedding of $X$ under $v$ and $s$ as

$$
H_A(X, v, s) \mapsto \veccat{X}{\frac{s-X^Tv_{1,\ldots,d}}{v_{d+1}}}
$$

where $\veccat{X}{y}$ is the set of vectors $x_i \in X$ extended by one dimension with coefficients $y_i \in \reals$.
Above and in the rest of this text, formulae that involve sets of vectors are per-vector operations.
One can easily see, that for any vector $x'$ in $H_A(X,v,s)$ holds $\mdot{x'}{v} = s$, i.e., that all points in $H_A(X,v,s)$ lie on an at least one-codimensional hyperplane.
$v$ is a normal vector of the hyperplane and whenever $s \neq 0$, the hyperplane is affine.
Using the affine hyperplane embedding, we can define the affine inversion sphere embedding $S_A(X, v, s) \subset \reals^{d+1}$~as

$$
S_A(X, v, s) \mapsto \frac{H_A(X,v,s)\hphantom{^2}}{\sqnorm{H_A(X,v,s)}}
$$

As per inversion theory, all points in $S_A(X, v, s)$ lie on a hypersphere that includes the origin in $\reals^{d+1}$ and whose center lies in direction $v$.
The furthest point to the origin in direction $v$ is precisely the inverted image of $sv$, i.e., $\frac{v}{s}$.
The center, thereby, is given as $\frac{v}{2s}$ and the radius is $\frac{1}{2s}$.
The definition of the affine inversion sphere embedding is analogous to the inverse stereographic projection in higher dimensions where $\frac{v}{s}$ denotes the north pole.
Accordingly, we can define the inversion-based spherical embedding $S(X, v, s)$ as

$$
S(X, v, s) \mapsto 2s\left(S_A(X, v, s) - \frac{v}{2s}\right) = 2s S_A(X, v, s) - v
$$

All vectors in $S(X, v, s)$ then must lie on a unit sphere around the origin in~$\reals^{d+1}$.
The choice of $v$ and $s$ both affect how the data is distributed on the resulting unit sphere.
Very small values in $s$ move the affine hyperplane closer to the origin, resulting in points clustering around $-v$, whereas very large values result in points clustering around $v$.
A more even spread of the data over the entire sphere is desirable in most cases.
One can find such an $s$ by grid search over a reasonable interval of $s$ values and investigate, e.g., the Angle-Based Intrinsic Dimensionality \citep[ABID,][]{Thordsen2022} of the resulting spherical distribution.
In case $X$ is very large, this can be estimated from a sample of $X$.
Mean-centering $X$ in advance and scaling $X$ to a mean vector norm of 1 typically results in a robust choice of $s$ within $[10^{-3},10^3]$ in which an exponentially spaced grid with about 50 steps should find a suitable value.
We can give the inverse of $S(X,v,s)$ as
$$
S^{-1}(Y,v,s) = 2s \left(\frac{Y + v\hphantom{^2}}{\sqnorm{Y + v}}\right)_{1,\ldots,d}
$$
where $(.)_{1,\ldots,d}$ denotes dropping the extra dimension $d+1$, such that $S^{-1}(S(X,v,s),v,s) = X$.

\section{Embedded Hyperspherical Caps and Hyperballs}

In this section, we first focus on the simpler case where $v=(0, \ldots, 0, 1)$, which clearly is a useful choice and preferred for applications where no other constraints have to be satisfied.
For that $v$, the hyperspherical cap resulting from intersecting a hyperplane with $S(X, v, s)$ corresponds to a hyperball in the original space of $X$ whenever $-v$ is not included in the cap.
It also simplifies the (un-)embedding functions to
\begin{eqnarray*}
	\simplified{S}(X,v,s) = 2s \frac{\veccat{X}{s}}{\sqnorm{X} + s^2} - v
	&\text{and}&
	\simplified{S}^{-1}(Y,v,s) = s \frac{Y_{1,\ldots,d}}{1 + Y_{d+1}}
	\,.
\end{eqnarray*}
To prove the relation between hyperspherical caps and hyperballs, we first introduce the necessary notation.

\begin{definition}[Hypherspherical Caps]
	The \emph{open hyperspherical cap}~$C$ of a directional unit-length vector $p \in \reals^d$ and a bias $b \in (0,1)$ is the set of points on the unit sphere, whose dot product with~$p$ is greater than~$b$, 
	$$
	C(p,b) := \left\{ x \in \nsphere{d-1} \mid \mdot{x}{p} > b \right\}
	\,.
	$$
	The separating hyperplane has normal vector $p$ and bias $b$.
	The \emph{closed hyperspherical cap} is defined analogously with $\mdot{x}{p} \geq b$.
	The \emph{boundary of the cap} is the difference between the closed and the open cap.
\end{definition}

\begin{definition}[Hyperball]
	The \emph{open hyperball}~$B$ of a center $c \in \reals^d$ and a radius $r \in \reals_{>0}$ is the set of points whose Euclidean distance to~$c$ is less than~$r$, i.e.,
	$$
	B(c,r) := \left\{ x \in \reals^d \mid \norm{x-c} < r \right\}
	\,.
	$$
	The \emph{closed hyperball} is defined analogously with $\norm{x-c} \leq r$.
	The \emph{boundary of the hyperball} is the difference between the closed and the open cap.
\end{definition}

We can now prove the duality of hyperspherical caps in the spherical embedding space and hyperballs in the original space.
It is, again, important that $-v$ is not included in the cap, i.e., that the directional vector~$p$ and bias~$b$ of the cap satisfy $b + p_{d+1} > 0$ and consequentially $\mdot{v}{p} < b$, since $\simplified{S}^{-1}(-v,v,s)$ is ill-defined and corresponds to inifinity in all possible directions in the original space.
From inversive geometry, it is well known, that hyperspheres that do not contain the center of inversion get mapped to hyperspheres, e.g., Theorem~3.4 of \cite{Lee2020GeometryFI}.
Since we shift the center of the hyperball by $v$, the center of inversion ends up at $-v$ in the projected space.
The boundaries of the cap is a hypersphere, thus when the cap does not intersect $-v$, the boundary in the original space is also a hypersphere.
When $-v$ is contained in the cap, the \enquote{interior} of the cap gets projected to the \enquote{outerior} of the hypersphere in original space, i.e. the complement of the contained hyperball.
Conversely, if $-v$ is neither on the boundary nor inside the cap, the cap gets projected to a hyperball in the original space.
Since all points in $H_A(X,v,s)$ share the same $(d+1)$-th coefficient, cutting that dimension does not affect the resulting hyperball.
While proving the existence of the resulting hyperball is rather trivial, finding the explicit values for center and radius is not.
The proof for the Cap-Ball-Duality is oriented along the boundary of the cap and ball, i.e., the hypersphere of intersection of the unit $d$-sphere $\nsphere{d}$ and the hyperplane defined by $p$ and $b$ and the hypersphere around $c$ with radius~$r$ in the original space.
The formulae for the center and the radius are not derived directly, but are based on the intuition that the image under $S$ must be shifted towards $v$ to obtain the cap direction vector.
Using that intuition, exact values can be derived, resulting in a closed form existence proof.

\begin{theorem}[Cap-Ball-Duality]
	For arbitrary $p \in \nsphere{d}$ and $b \in (-1,1)$ with $b + p_{d+1} > 0$, the image under $S^{-1}(\blank, v, s)$ of the hyperspherical cap $C(p,b)$ is a hyperball in $\reals^d$ whenever $v = (0,\ldots,0,1) \in \reals^{d+1}$ and $s \in \reals_{>0}$.
	The center and radius of the hyperball are
	$$
	c = S^{-1}\left(\frac{p - \alpha v}{\norm{p - \alpha v}}, v, s\right)
	,~\text{and}~
	r = s\sqrt{\frac{2\alpha}{b+p_{d+1}}}
	,~\text{where}~
	\alpha := \frac{1-b^2}{2(b+p_{d+1})}
	\,.
	$$
	\label{theorem:cap_ball_duality}
\end{theorem}
\begin{proof}
	The proof is based on showing that points on the boundary of the cap are projected to points on the boundary of the ball.
	For that, we can choose an arbitrary point on the boundary, start with the distance of its projection under $S^{-1}$ to the center of the ball, and show that it equals the radius of the ball specified in the Theorem.
	By continuity, all points inside/outside the cap get projected to points inside/outside the ball.
	The full proof can be found in the Appendix, since the derivation is quite technical and lengthy.
\end{proof}

With the simplifying choice of $v=(0,\ldots,0,1)$, we can give an alternative representation of $c$, that does not include $S$ or $S^{-1}$ as
\begin{eqnarray}
	c &=& s\,\frac{p_{1,\ldots,d}}{\sqrt{1-p_{d+1}^2 + (p_{d+1}-\alpha)^2} + p_{d+1} - \alpha}.
	\label{eq:ball_center_simplified}
\end{eqnarray}
This \enquote{simplified} equation highlights that the center $c$ of the hyperball is collinear to the first $d$ coefficients of the normal vector $p$ of the cap.
With that observation, we can choose $\beta$ such that $\beta c = p_{d+1}$ and derive the opposite direction to find a hyperspherical cap in the embedding space given a hyperball.
\begin{theorem}[Cap-Ball-Duality 2]
	For arbitrary center~$c \in \reals^d$ and radius~$r \in \reals_{>0}$, the image under $S(\blank, v, s)$ of the hyperball $B(c,r)$ is a hyperspherical cap on $\nsphere{d}$ whenever $v = (0,\ldots,0,1) \in \reals^{d+1}$ and $s \in \reals_{>0}$.
	The direction~$p$ and bias~$b$ of the hyperspherical cap are
	$$
	p = \left(\beta c, \sqrt{1-\beta^2\sqnorm{c}}\right)
	\,, \text{ and }
	b = \frac{s \sqrt{s^2 + \beta^2\sqnorm{c}r^2} - r^2\sqrt{1-\beta^2\sqnorm{c}}}{r^2+s^2}
	\,,
	$$
	$$
	\text{where }
	\beta := \frac{2s}{\sqrt{(\sqnorm{c}+r^2+s^2)^2-4\sqnorm{c}r^2}}
	\,.
	$$
	\label{theorem:cap_ball_duality2}
\end{theorem}
\begin{proof}
	The proof follows from inverting the equations in \reftheorem{theorem:cap_ball_duality} and inserting the chosen definitions here to derive the proposed value of $\beta$.
	The full proof can again be found in the Appendix.
\end{proof}

Using the two theorems, we can derive some auxiliary measures, that in applications will likely be of use.
First, we can show that for $s \geq \max_{x\in X} \norm{x}$ all points will be projected to the (closed) hemi-hypersphere in direction $v$.

\begin{corollary}[Hemispherical Embedding]
	Given indexed data~$x_i \in X \subset \reals^d$ and $v=(0,\ldots,0,1) \in \reals^{d+1}$.
	For any large enough scale $s \geq \max_{x \in X} \Vert x \Vert$, $S(X,v,s)$ lies on the closed hemisphere in direction $v$, i.e., $\forall \hat{x} \in S(X,v,s)\colon \mdot{\hat{x}}{v} \geq 0$.
	For any $s > \max_{x \in X} \Vert x \Vert$, all embedded vectors lie on the open hemisphere.
\end{corollary}
\begin{proof}
	Choosing $p := v$ and $b := 0$, we obtain a hyperball with center at the origin and radius $s$.
	Iff $s \geq \max_{x \in X} \norm{x}$, then $\mdot{x}{v} \geq 0$ for all $x \in X$, i.e., all points in $S(X,v,s)$ lie on the closed hemi-hypersphere in direction $v$.
	Iff $s > \max_{x \in X} \norm{x}$, then $\mdot{x}{v} > 0$ for all $x \in X$, i.e., all points in $S(X,v,s)$ lie on the open hemi-hypersphere in direction $v$.
\end{proof}

Next, we can give closed forms for the inner product and squared Euclidean distance after (un-)embedding in terms of the same measures in the original space.
This opens the embedding up to kernel-based applications, where the spherical embedding is implicitly applied to the kernel space and for corrections in distance computations to account for the non-isometry of the embedding.
The distance equations can also be used to compute distances to, e.g., the hyperball centers implicitly in the embedding space.

\begin{corollary}
	Given any $s \in \reals_{>0}$ and assuming $v = (0,\ldots,0,1) \in \reals^{d+1}$.
	Let $x, y \in \reals^{d}$ and $\hat{x} = S(x,v,s), \hat{y} = S(x,v,s)$, then the following equations hold:
	\begin{align}
		\mdot{\hat{x}}{\hat{y}}
		&=
		1-2s^2\frac{
			\mdot{x}{x} -2\mdot{x}{y} + \mdot{y}{y}
		}{(\mdot{x}{x}+s^2)(\mdot{y}{y}+s^2)}
		=1-\frac{2s^2\sqnorm{x-y}}{(\sqnorm{x}+s^2)(\sqnorm{y}+s^2)}\\
		\sqnorm{\hat{x}-\hat{y}}
		&=
		\frac{4s^2\sqnorm{x-y}}{(\sqnorm{x}+s^2)(\sqnorm{y}+s^2)}\\
		\mdot{x}{y}
		&= s^2\frac{\mdot{\hat{x}}{\hat{y}} - \hat{x}_{d+1}\hat{y}_{d+1}}{(1+\hat{x}_{d+1})(1+\hat{y}_{d+1})}
		= s^2\frac{\mdot{\hat{x}}{\hat{y}} - \mdot{\hat{x}}{v}\mdot{\hat{y}}{v}}{(1+\mdot{\hat{x}}{v})(1+\mdot{\hat{y}}{v})}\\
		\sqnorm{x-y}
		&=
		\frac{s^2}{(1+\hat{x}_{d+1})(1+\hat{y}_{d+1})}\sqnorm{\hat{x}-\hat{y}}\\
		&= \frac{4s^2}{(4-\sqnorm{\hat{x}-v})(4-\sqnorm{\hat{y}-v})}\sqnorm{\hat{x}-\hat{y}}
	\end{align}
	\label{corollary:kernel_functions}
\end{corollary}
\begin{proof}
	The proofs consist of inserting the definitions and simplifying the resulting equations.
	The derivations are given in the Appendix.
\end{proof}

For moving from the spherical to the Euclidean space, we provide two formulae each, one using the last coordinate and one using the direction vector.
The formulae using the last coordinate is computationally more efficient, yet, the other formulae are a better fit for kernel-based applications.
The latter in conjunction with Theorem \ref{theorem:cap_ball_duality} also provide exact solutions for arbitrary $v$ when not discarding the last coordinate upon unembedding.
When discarding the last coordinate, as defined in $S^{-1}(\blank, v, s)$, the resulting balls -- already lying on a hyperplane orthogonal to $v$ -- get linearly projected to a slanted hyperplane -- orthogonal to $(0,\ldots,0,1)$ -- resulting in one of the dimensions getting \enquote{squished}.

\section{Embedded Hyperspherical Caps and Hyperellipsoids}

For embedding directions $v$ other than $(0,\ldots,0,1)$, the inverse image of a hyperspherical cap under $S^{-1}(\blank, v, s)$ is an ellipsoid, with all but one identical radii (semi-axes).
The single other radius is smaller and its axis is collinear to $v_{1,\ldots,d}$.

\begin{theorem}[Cap-Ellipsoid-Duality]
	For arbitrary $p,v \in \reals^{d+1}$ and $b \in \reals_{\geq 0}$ with $\norm{v} = 1$ and $b + \mdot{p}{v} > 0$, the image under $S^{-1}(\blank, v, s)$ of the hyperspherical cap $C(p,b)$ is a hyperellipsoid in $\reals^d$ whenever $v \neq (0,\ldots,0,1)$ and $s \in \reals_{>0}$.
	The center of the hyperellipsoid is
	$$
	c = S^{-1}\left(\frac{p - \alpha v}{\norm{p - \alpha v}}, v, s\right)
	~\text{where}~
	\alpha := \frac{1-b^2}{2(b+\mdot{p}{v})}.
	$$
	The smallest radius and the corresponding direction are
	$$
	r_1 = s\sqrt{\frac{(1-b^2)v_{d+1}^2}{(b+\mdot{p}{v})^2}}
	= s\sqrt{\frac{(1-b^2)}{(b+\mdot{p}{v})^2}} \vert v_{d+1} \vert
	,~\text{and}~
	a_1 = \frac{v_{1,\ldots,d}}{\norm{v_{1,\ldots,d}}}
	$$
	and the radius in all directions $a_2$ orthogonal to $a_1$ is
	$$
	r_2 = s\sqrt{\frac{1-b^2}{(b+\mdot{p}{v})^2}}.
	$$
\end{theorem}
\begin{proof}
	The proof follows from Theorem \ref{theorem:cap_ball_duality}, since the hyperspherical cap $C(p,b)$ is the image under $S$ of the hyperball $B(c',r)$ in $H_A(\reals^d,v,s)$ where the first $d$ coefficients of $c'$ equal $c$.
	The hyperplane on which all points in $H_A(\reals^d,v,s)$ live is orthogonal to $v$ and the final step of the inversive projection is to discard the last coefficient of the embedded vectors, which can be represented as a linear projection on the hyperplane with normal vector $(0,\ldots,0,1)$.
	The cosine between these two hyperplanes is $\mdot{v}{(0,\ldots,0,1)} = v_{d+1}$ and all directions orthogonal to $v$ are shared between the two hyperplanes.
	Accordingly, the radius in the direction of $v$ has to be scaled by the cosine between the two hyperplanes $v_{d+1}$, while the radius in all other directions remains unchanged.
	A visual proof is given in Figure \ref{fig:cap_ellipsoid_duality}.
\end{proof}
\begin{figure}
	\centering
	\tdplotsetmaincoords{80}{70}
	\begin{tikzpicture}[scale=2, tdplot_main_coords]
		\draw[thick,->] (0,0,0) -- (4.5,0,0) node[anchor=north]{$i$};
		\draw[thick,->] (0,0,0) -- (0,4,0) node[anchor=south]{$j$};
		\draw[thick,->] (0,0,0) -- (0,0,1) node[anchor=south]{$d+1$};
		\pgfmathsetmacro{\nbx}{1.5}
		\pgfmathsetmacro{\nby}{1.25}
		\pgfmathsetmacro{\nbz}{0}
		\pgfmathsetmacro{\vangle}{-30}
		\pgfmathsetmacro{\vy}{sin(\vangle)}
		\pgfmathsetmacro{\vz}{sqrt(1-\vy*\vy)}
		\pgfmathsetmacro{\halfside}{1.5}
		\pgfmathsetmacro{\radius}{1.25}
		\pgfmathsetmacro{\vplanelift}{\halfside * sqrt(1/(\vz*\vz)-1)}
		\draw[draw=none,fill=orange,opacity=0.2] (\nbx-\halfside,\nby-\halfside,\nbz) -- (\nbx+\halfside,\nby-\halfside,\nbz) -- (\nbx+\halfside,\nby+\halfside,\nbz) -- (\nbx-\halfside,\nby+\halfside,\nbz) -- cycle;
		\draw[draw=none,fill=blue,opacity=0.2] (\nbx-\halfside,\nby-\halfside,-\vplanelift) -- (\nbx+\halfside,\nby-\halfside,\nbz-\vplanelift) -- (\nbx+\halfside,\nby+\halfside,\vplanelift) -- (\nbx-\halfside,\nby+\halfside,\nbz+\vplanelift) -- cycle;
		\draw[dashed] (\nbx-\halfside,\nby,\nbz) -- (\nbx+\halfside,\nby,\nbz);
		\draw[dashed] (\nbx,\nby-\halfside,\nbz) -- (\nbx,\nby+\halfside,\nbz);
		\draw[dashed] (\nbx,\nby-\halfside,\nbz-\vplanelift) -- (\nbx,\nby+\halfside,\nbz+\vplanelift);
		\coordinate (NormalBase) at (\nbx,\nby,\nbz);
		\tdplotsetrotatedcoordsorigin{(NormalBase)}
		\tdplotsetrotatedcoords{0}{90}{0}
		\tdplotdrawarc[tdplot_rotated_coords,gray,opacity=.5]{(0,0,0)}{\radius*\vz}{0}{360}{}{};
		\tdplotsetrotatedcoords{0}{0}{0}
		\draw[thick,->,blue!70!black] (NormalBase) -- ($(NormalBase) + (0,\radius*\vy,\radius*\vz)$) node[anchor=south]{$v$};
		\draw[thick,->,blue!70!black,dashed] (NormalBase) -- ($(NormalBase) + (0,\radius*\vy,0)$) node[anchor=north west,yshift=.15em]{$v_{1,\ldots,d}$};
		\draw[thick,-,blue!70!black] ($(NormalBase) + (0,-.05,\radius*\vz)$) -- ($(NormalBase) + (0,.05,\radius*\vz)$);
		\node[blue!70!black,anchor=west,yshift=.2em] at ($(NormalBase) + (0,0,\radius*\vz)$) {$v_{d+1}$};

		\draw[blue!70!black,dotted] ($(NormalBase) + (0,\radius*\vy,0)$) -- ($(NormalBase) + (0,\radius*\vy,\radius*\vz)$);
		\draw[blue!70!black,dotted] ($(NormalBase) + (0,0,\radius*\vz)$) -- ($(NormalBase) + (0,\radius*\vy,\radius*\vz)$);
		\draw[thick,->,orange!70!black] (NormalBase) -- ($(NormalBase) + (0,0,\radius)$) node[anchor=south]{$(0,\ldots,0,1)$};
		\tdplotsetrotatedcoordsorigin{(NormalBase)}
		\tdplotsetrotatedcoords{0}{90}{0}
		\tdplotdrawarc[tdplot_rotated_coords,orange!50!blue!90!black,thick]{(0,0,0)}{.8}{180}{180-\vangle}{anchor=north,xshift=.15em}{$\theta$};
		\tdplotdrawarc[tdplot_rotated_coords,orange!50!blue!50!green!80!black,thick]{(0,0,0)}{.8}{180}{90-\vangle}{anchor=north east,xshift=1em,yshift=-.7em}{$\frac{\pi}{2}{-}\theta$};
		\tdplotdrawarc[tdplot_rotated_coords,orange!50!blue!90!black,thick]{(0,0,0)}{.8}{90}{90-\vangle}{anchor=east,yshift=-.15em}{$\theta$};
		\tdplotsetrotatedthetaplanecoords{90-\vangle}
		\pgfmathsetmacro{\squishedradius}{\radius * \vz}
		\tdplotdrawarc[tdplot_rotated_coords]{(NormalBase)}{\radius}{0}{360}{}{}
		\draw (NormalBase) ellipse ({\radius} and {\squishedradius});
		\pgfmathsetmacro{\nLines}{50}
		\pgfmathsetmacro{\stepAngle}{360/\nLines}
		\foreach \i in {1,...,\nLines}{
			\pgfmathsetmacro{\angle}{\i * \stepAngle}
			\pgfmathsetmacro{\cosAngle}{cos(\angle)}
			\pgfmathsetmacro{\sinAngle}{sin(\angle)}
			\pgfmathsetmacro{\x}{\radius * \sinAngle}
			\pgfmathsetmacro{\y}{\squishedradius * \cosAngle}
			\draw[dotted,opacity=.5] ($(NormalBase)+(\x,\y,0)$) -- ($(NormalBase)+(\x,\y,-\radius * \vy * \cosAngle)$);
		};
		\pgfmathsetmacro{\nbv}{\nby*\vy+\nbz*\vz}
		\draw[<-,thick,blue!70!black] (0,0,0) -- (0,\nbv*\vy,\nbv*\vz) node[midway,anchor=north east] {$sv$};
		\coordinate (NormalBase2) at (0,\nbv*\vy,\nbv*\vz);
		\tdplotsetrotatedcoordsorigin{(NormalBase2)}
		\tdplotsetrotatedcoords{0}{90}{0}
		\tdplotdrawarc[tdplot_rotated_coords,blue!70!black]{(0,0,0)}{.15}{180-\vangle}{270-\vangle}{yshift=-.1em,xshift=.25em}{$\cdot$};
		\draw[fill] ($(NormalBase) + (+\radius,0,0)$) circle (.5pt);
		\draw[fill] ($(NormalBase) + (-\radius,0,0)$) circle (.5pt);
		\draw[fill] ($(NormalBase) + (0,+\radius*\vz,0)$) circle (.5pt);
		\draw[fill] ($(NormalBase) + (0,-\radius*\vz,0)$) circle (.5pt);
	\end{tikzpicture}
	\caption{
		A visual proof that the radius in direction $v_{1,\ldots,d}$ is the only one varying from the radius in the spherical case and that its scaling factor is $v_{d+1}$.
		$i$ and $j$ indicate arbitrary basis vectors orthogonal to $(0,\ldots,0,1)$, not necessarily parallel to coordinate axes.
		The cosine of $\theta$ equals $v_{d+1}$.
	}
	\label{fig:cap_ellipsoid_duality}
\end{figure}
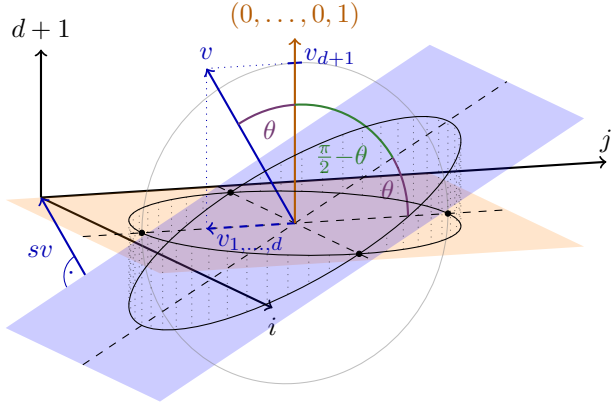

The Cap-Ball-Duality can then be interpreted as the limit case of the Cap-Ellipsoid-Duality, where the difference between the \enquote{squished} and the other dimensions vanishes and the \enquote{squished} axis therefore becomes arbitrary and ill-defined.
We neither provide formulae for the inverse direction of the duality, nor for dot products, and squared distances for the Cap-Ellipsoid-Duality.
Whilst those functions should have a closed form, they are likely very large and cover a much less useful case of ellipsoids with exactly one axis of different length.
The inclined reader can likely guess appropriate functions exploiting the observation, that the original space is a slanted view of the spherical case and distances are compressed along the \enquote{squished} axis, but proofing these guessed functions is certainly cumbersome.
Otherwise, one can avoid the projection to one less dimension in the unembedded space and use the Cap-Ball-Duality instead.
When moving from spherical data to a non-spherical representation, that may be a sufficient representation and the formulae for products and distances for the Cap-Ball-Duality apply.
Since the unembedding scales the space along the embedding direction, it should be possible to chain multiple embeddings to achieve an equivalent duality between general ellipsoids and hyperspherical caps.

\begin{hypothesis}[Generalized Cap-Ellipsoid-Duality]
	For arbitrary ellipsoids $E \subseteq \reals^d$ with axes $a_1,\ldots,a_d$ and radii $r_1,\ldots,r_d$ such that $r_i \leq r_d$ for all $i \in \{1,\ldots,d-1\}$, there exist $v_1,\ldots,v_{d-1}, s_1,\ldots,s_{d-1}$, and $b$ such that the repeated embedding of $E$ is the hyperspherical cap $C(v_{d-1}, b)$, i.e.,
	\begin{equation*}
		S(\ldots(S(E, v_1, s_1),\ldots),v_{d-1},s_{d-1}) = C(v_{d-1}, b)
		\,.
	\end{equation*}
	Further, $a_i = (v_i)_{1,\ldots,d}$ for all $i$ and a closed form for the center exists.
\end{hypothesis}

\section{Parameter Choice}
To embed the data with the proposed embedding vector $v = (0,\ldots,0,1)$, we still need to choose the scaling parameter~$s$.
The choice of that parameter controls the distribution of samples on the resulting hypersphere.
For infinitesimal~$s$, the samples cluster around~$-v$ whereas for infinite~$s$, the samples cluster around~$v$.
Ideally for most down-stream applications, the samples would \enquote{evenly distribute} over the sphere as to, e.g., not provoke numerical instabilities.
As a criterion for a good spread of the data, we propose to consider the intrinsic dimensionality of the embedded data.
The intrinsic dimensionality of the data is the number of dimensions that are necessary to represent the space populated by the data.
We consider that measure as a good criterion, since it assigns the same value to the worst case on both ends of the value space of~$s$ and a higher value everywhere inbetween.
Intuitively, targeting a high intrinsic dimensionality can be motivated by best preserving the data complexity, as a reduction in ID indicates a loss of structural information.
If the entire dataset collapses to $v$ or $-v$ on the sphere, the intrinsic dimensionality should be minimum.
For any point inbetween, the intrinsic dimensionality should be larger with a likely maximum hopefully close to the intrinsic dimensionality of the original space.
We expect that maximum to provide the best spreading of the data over the sphere and therefore be the best choice for~$s$.
Since the embedded data will lie on a sphere, we propose to employ the ABID estimator of \cite{Thordsen2022}, which considers the pairwise cosine distribution of samples.
The ABID estimator is typically applied in a local manner, i.e., evaluated over a small neighborhood around each point in the sample, commonly using the $k$-nearest neighbors.
We here instead suggest to use it as a global measure of intrinsic dimensionality, by aggregating it over the entire (embedded) dataset.
That implies that we consider the dataset to populate a linear subspace of the original space.
We thereby do not explicitly attempt to cover the local structure of the data, but rather the global structure.
The ABID estimator of \cite{DBLP:conf/sisap/ThordsenS23} is defined as
\begin{align}
	\mathrm{ABID}(X) = \mathbb{E}_{x,y \in X}\left[ \cos(x,y)^2 \right]^{-1}
\end{align}
We are then interested in the value of $s$ that maximizes the ABID estimate evaluated over the embedded samples.
We expect the ABID of embedded samples to be upper bounded by the ABID value on the unembedded space plus one.
We will later provide empirical evidence that these values can be expected to be identical.
In the unembedded space, we propose to mean-center the dataset, such that the cosines are computed relative to the mean of the dataset.
This should provide an as-large-as-possible value for the ABID estimate, since if all points were shifted far away from the origin, all cosines would be maximized.
The data would then be almost indistinguishable from a single point.
In the embedded space, the origin, i.e., the center of the sphere, will be the point relative to which the cosines are computed.
That results in minimum values of $1$ for the ABID estimate, since the center of the sphere virtually becomes part of the dataset, which then lies on the line through the origin in direction of $v$.
Solving the criteria for an optimal $s$ analytically, however, requires intricate knowledge of the distribution of the dataset, which we would rather avoid, since it limits the parameter choice to specific distributions.
Instead, we perform a parameter sweep for $s$ over a reasonable range.
We propose to center this range on the mean vector norm of the samples in the original space.
Given two samples $x$ and $y$ and their respective embedded points $\widehat{x}$ and $\widehat{y}$, the ratio of cosines prior and posterior to embedding equates
\begin{align*}
	\frac{\cos(x,y)}{\cos(\widehat{x},\widehat{y})}
	&\parbox{5em}{\centering$=$} \frac{
		\frac{\sqnorm{x} + \sqnorm{y} - \sqnorm{x-y}}{2\norm{x}\norm{y}}
	}{
		\frac{2-\sqnorm{\widehat{x}-\widehat{y}}}{2}
	}
	\shortintertext{Which when inserting the solution for $\sqnorm{\widehat{x}-\widehat{y}}$ reduces to}
	&\parbox{5em}{\centering$=$} \frac{
		\sqnorm{x} + \sqnorm{y} - \sqnorm{x-y}
	}{
		\left(2-\frac{4s^2\sqnorm{x-y}}{\left(\sqnorm{x}+s^2\right)\left(\sqnorm{y}+s^2\right)}\right)
		\norm{x}\norm{y}
	}\\
	&\parbox{5em}{\centering$=$} \frac{
		\left(\sqnorm{x} + \sqnorm{y} - \sqnorm{x-y}\right)
		\left(\sqnorm{x}+s^2\right)\left(\sqnorm{y}+s^2\right)
	}{
		\bigg(
			2\left(\sqnorm{x}+s^2\right)\left(\sqnorm{y}+s^2\right)
			-4s^2\sqnorm{x-y}
		\bigg)
		\norm{x}\norm{y}
	}
	\shortintertext{
		The distribution of that term clearly depends on the distribution of the samples in the original space.
		If we assert $s$ to be the mean vector norm of the samples and approximate $\norm{x}$ and $\norm{y}$ by the same value, the term further simplifies to
	}
	&\parbox{5em}{$\centering\stackrel{\norm{x},\norm{y} \to s}{=}$} \frac{
		\left(s^2+s^2-\sqnorm{x-y}\right)
		\left(s^2+s^2\right)\left(s^2+s^2\right)
	}{
		\left(
			2\left(s^2+s^2\right)\left(s^2+s^2\right)
			-4s^2\sqnorm{x-y}
		\right)
		s^2
	}\\
	&\parbox{5em}{\centering$=$} \frac{
		\left(2s^2-\sqnorm{x-y}\right)
		4s^4
	}{
		\left(
			8s^4
			-4s^2\sqnorm{x-y}
		\right)
		s^2
	} = 1
\end{align*}
So as all vector norms approach $s$, the ratio of all cosines approaches $1$ as well.
Accordingly, in the limit, the distribution of cosines in the embedded space should be identical to the distribution of cosines in the unembedded space and thereby the ABID estimate should be identical.
The limit case is of course likely not observed, whereby that choice of $s$ is only approximately optimal.
In our experiments, however, this choice for $s$ was always close to the optimal value and off by less than a multiple of $10$.
Accordingly, a logarithmically spaced grid from $0.1$ to $10$ times the mean vector norm of the samples in the original space should be sufficient to cover the optimal value of $s$ and a small number of grid values, e.g. 20, should suffice to find a suitable value.
In case of outliers with a very large vector norm, the initial value can be adjusted to the median vector norm.
As can be seen in \reffig{fig:embedding_lid}, the ABID estimates are capped at one plus the dimensionality of the original space for synthetic datasets.
The corresponding pairwise cosines before and after embedding are visualized in \reffig{fig:cosine_ratio} and approximate the identity function.

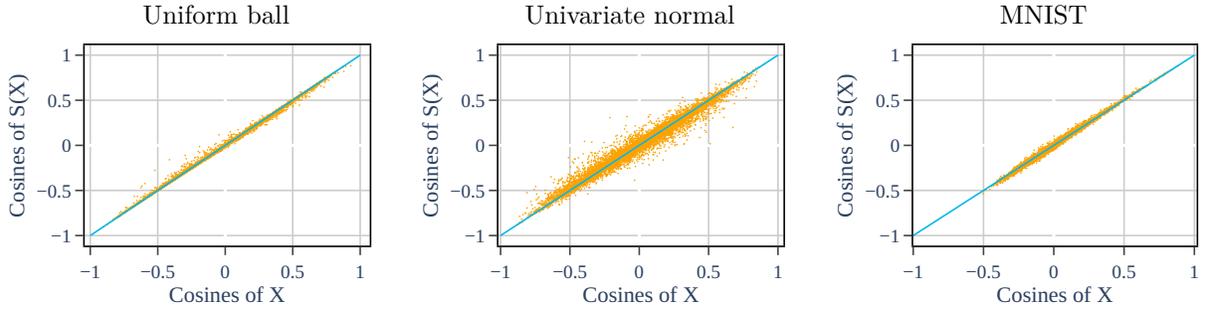
\begin{figure}
	\centering
	\begin{tikzpicture}
		\node (A) at (0,0) {\includegraphics[scale=.8]{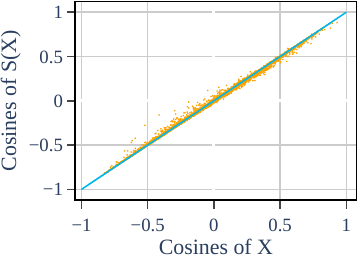}};
		\node (B) at (5.5,0) {\includegraphics[scale=.8]{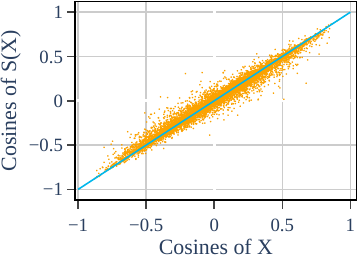}};
		\node (C) at (11,0) {\includegraphics[scale=.8]{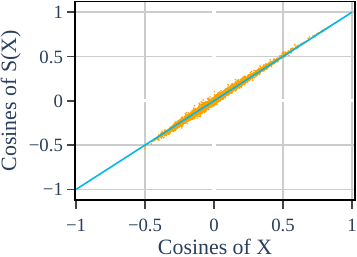}};
		\node[anchor=south, xshift=1em] at (A.north) {Uniform ball};
		\node[anchor=south, xshift=1em] at (B.north) {Univariate normal};
		\node[anchor=south, xshift=1em] at (C.north) {MNIST};
	\end{tikzpicture}
	\caption{
		Pairwise cosines before and after spherical embedding of mean-centered data.
		From each distribution we drew 500 samples.
		For the generated datasets we used 10 dimensional distributions.
		The visualized cosines belong to 10k random sample pairs.
		The blue line indicates the identity function.
	}
	\label{fig:cosine_ratio}
\end{figure}

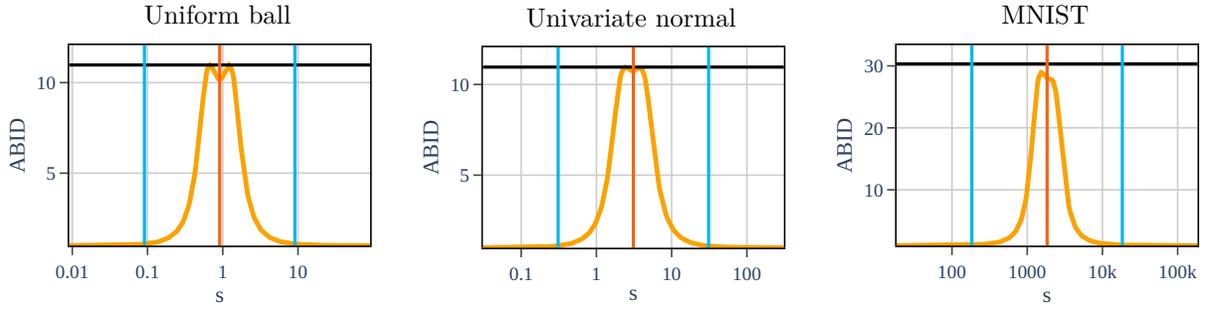
\begin{figure}
	\centering
	\begin{tikzpicture}
		\node (A) at (0,0) {\includegraphics[scale=.8]{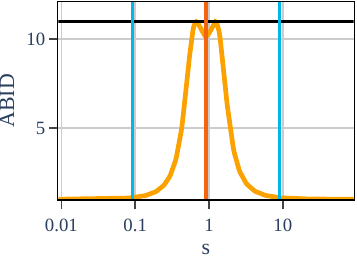}};
		\node (B) at (5.5,0) {\includegraphics[scale=.8]{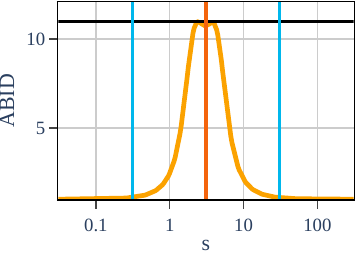}};
		\node (C) at (11,0) {\includegraphics[scale=.8]{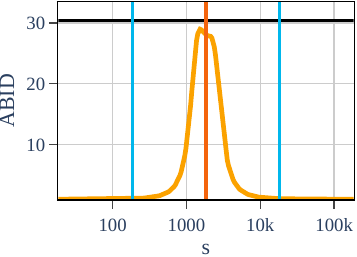}};
		\node[anchor=south, xshift=1em] at (A.north) {Uniform ball};
		\node[anchor=south, xshift=1em] at (B.north) {Univariate normal};
		\node[anchor=south, xshift=1em] at (C.north) {MNIST};
	\end{tikzpicture}
	\caption{
		ABID estimates over spherical embeddings with varying $s$ parameter.
		From each distribution we drew 500 samples.
		For the generated datasets we used 10 dimensional distributions.
		The horizontal black line indicates the ABID estimate of the original dataset plus one.
		The red vertical line indicates the mean vector norm.
		The blue vertical lines indicate $0.1$ and $10$ times the mean vector norm, respectively.
	}
	\label{fig:embedding_lid}
\end{figure}

\section{Relevance to Machine Learning}
The Cap-Ball-Duality provides an algorithmic solution to apply sphere-based algorithms to general Euclidean space and to translate between separating hyperspheres and hyperplanes.
The former allows to expand spherical algorithms to general Euclidean space, such as ScaNN of \cite{ScaNN} or HIOB of \cite{DBLP:conf/sisap/ThordsenS23}, and to solve ball-based problems with hyperplanes, such as the minimal bounding sphere problem.
The latter allows to exchange the separating geometry, e.g., in support vector machines or spatial indexes, to work with balls instead of hyperplanes and vice-versa.
The results from these adaptations are increasingly approximate for small biases/large radii, but allow for a rapid prototyping or cheap approximation of algorithms in the respective other geometry.

\begin{figure}
	\centering
	\begin{tikzpicture}
		\node (A) at (0,0) {\includegraphics[scale=.8]{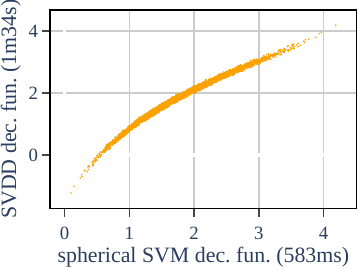}};
		\node (B) at (8,0) {\includegraphics[scale=.8]{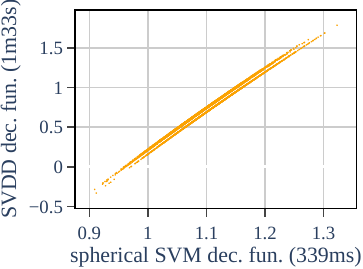}};
		\node[anchor=south, xshift=1em] at (A.north) {1 Gaussian};
		\node[anchor=south, xshift=1em] at (B.north) {3 Gaussians};
	\end{tikzpicture}
	\caption{
		The decision function values for SVDDs and SVMs fitted to the original and embedded data, respectively.
		Left shows the results for a Gaussian distribution in 10 dimensions, right for three Gaussian blobs in 10 dimensions.
		The datasets of 5000 points each were iteratively shifted to the center of the SVM induced sphere to obtain an equal shift in distances in all directions.
		The shifting process is included in the displayed fitting times.
	}
	\label{fig:svm_svdd}
\end{figure}
\begin{figure}
	\centering
	\begin{tikzpicture}
		\node[anchor=north west] (A) at (0,0) {\includegraphics[scale=.75]{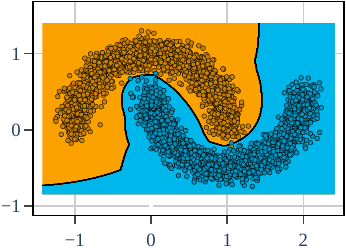}};
		\node[anchor=north west] (B) at (4.75,0) {\includegraphics[scale=.75]{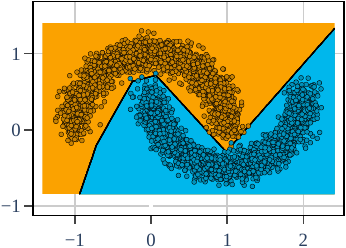}};
		\node[anchor=north west] (C) at (9.5,0) {\includegraphics[scale=.75]{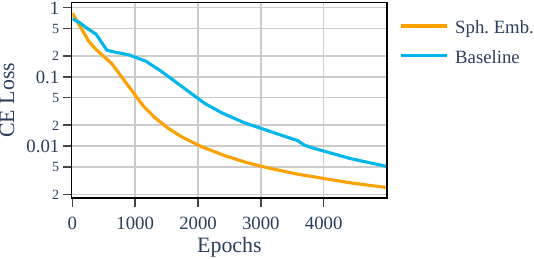}};
		\node[anchor=south, xshift=1em] at (A.north) {Spherical Embedding};
		\node[anchor=south, xshift=1em] at (B.north) {Baseline};
		\node[anchor=south, xshift=-1em] at (C.north) {Training Loss};
	\end{tikzpicture}
	\caption{
		The decision boundaries and training losses for neural networks trained with spherical embedding and ReLU as first activation, respectively.
		The baseline network consisted of three fully connected layers, with shapes $2 \times 2$, $2 \times 16$ and $16 \times 2$ with ReLU activations inbetween.
		The spherical embedding network used the $S$ with a learned $s$ as first activation and therefore had a $3 \times 16$ second layer.
	}
	\label{fig:neural_nets}
\end{figure}
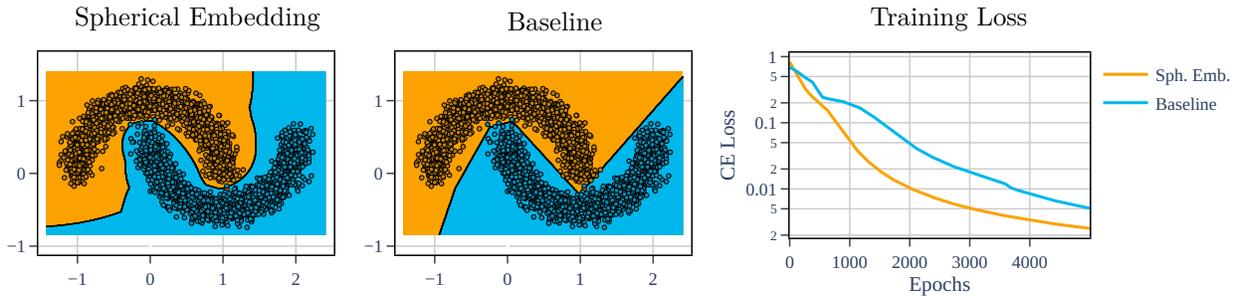
\begin{figure}
	\caption{
		Recall of ScaNN and Annoy spatial indexes for spherical and non-spherical data before and after (un)embedding.
		Highlighted recall values are the best for each dataset and index.
		Shaded cells indicate values that we would have expected to improve over the \enquote{Direct} columns.
	}
	\label{table:ann_results}
	\begin{center}
		\begin{tabular}{|l|c|c|c|c|c|c|c|c|}
			\hline
			\multirow{2}{*}{Dataset} & \multirow{2}{*}{d} & \multirow{2}{*}{Spherical} & \multicolumn{3}{c|}{ScaNN $10@10$-recall} & \multicolumn{3}{c|}{Annoy $10@10$-recall} \\ \cline{4-9} 
			 & & & Direct & Emb. & Unemb. & Direct & Emb. & Unemb. \\ \hline
			Normalized Blobs & 50 & Yes & 72.78\% & \textbf{72.96\%} & 26.83\% & 35.70\% & 35.99\% & \cellcolor[HTML]{F4E5CC}\textbf{41.79\%} \\ \hline
			Normalized Blobs & 100 & Yes & \textbf{80.95\%} & \textbf{80.95\%} & 26.73\% & 26.10\% & 26.38\% & \cellcolor[HTML]{F4E5CC}\textbf{30.18\%} \\ \hline
			GloVe & 100 & Yes & 89.95\% & \textbf{90.08\%} & 54.84\% & 7.13\% & 7.15\% & \cellcolor[HTML]{F4E5CC}\textbf{16.84\%} \\ \hline
			NY Times & 256 & Yes & \textbf{81.11\%} & 80.80\% & 61.44\% & 14.50\% & 14.82\% & \cellcolor[HTML]{F4E5CC}\textbf{21.52\%} \\ \hline
			Last.fm & 64 & Yes & \textbf{87.93\%} & 86.35\% & 32.32\% & 16.99\% & \textbf{18.44\%} & \cellcolor[HTML]{F4E5CC}3.89\% \\ \hline
			Blobs & 50 & No & 38.93\% & \cellcolor[HTML]{F4E5CC}\textbf{70.04\%} & 28.13\% & \textbf{41.54\%} & 34.23\% & 40.88\% \\ \hline
			Blobs & 100 & No & 37.38\% & \cellcolor[HTML]{F4E5CC}\textbf{74.99\%} & 25.58\% & 30.11\% & 27.25\% & \textbf{30.48\%} \\ \hline
			Fashion MNIST & 784 & No & \textbf{99.85\%} & \cellcolor[HTML]{F4E5CC}75.69\% & 65.82\% & \textbf{85.59\%} & 66.84\% & 65.57\% \\ \hline
			SIFT & 128 & No & \textbf{98.39\%} & \cellcolor[HTML]{F4E5CC}84.05\% & 55.90\% & \textbf{35.74\%} & 26.64\% & 32.88\% \\ \hline
			\end{tabular}
	\end{center}
\end{figure}

Using a linear SVM to separate spherically embedded data from $-v$, for example, gives a cheap approximation for the order of scores of the Support Vector Data Descriptor \citep[SVDD]{DBLP:journals/ml/TaxD04} -- an adaptation of the SVM designed to fit a bounding sphere around the data.
The strong correlation between these two measures as well as the significant speedup is displayed in \reffig{fig:svm_svdd}.
While the QP-solver-based SVDD implementation\footnote{\url{https://github.com/iqiukp/SVDD-Python}} required about 1.5 minutes for each of the tasks, the spherical embedding-based SVM required only half a second to solve multiple SVM tasks for proper fitting.
The results of the SVM and SVDD are not identical, but the order of the scores is mostly preserved.

Similarly, the performance of spatial indexes for approximate nearest neighbor search can be improved by enforcing the geometry on which they work best.
The (un)embedding of course introduces an approximation error, though we expect it to be small enough for larger datasets to not significantly impact the recall of already approximate approaches.
The ScaNN index \citep{ScaNN}, for example, works best for cosine distance, since it is based on product quantization.
By embedding the data onto a hypersphere, we can approximately translate the Euclidean distances of non-spherical data into cosine distances.
The Annoy index by Bernhardsson\footnote{\url{https://github.com/spotify/annoy}} is based on projections to random interpoint axes in the dataset.
That approach essentially quantizes the squared Euclidean distance and we expect it to work better on non-spherical data.
\reftable{table:ann_results} shows the $10@10$-recall of the ScaNN and Annoy indexes on various datasets with and without spherical embedding.
The datasets were mostly taken from the ANN-Benchmarks\footnote{\url{https://ann-benchmarks.com/}} and toy datasets of 10 (normalized) gaussian blobs were added.
Non-spherical datasets were mean centered in advance.
The spherical embeddings were computed with $s$ equal to the mean vector norm.
For the spherical unembeddings the datasets were first normalized, potentially losing information on non-spherical datasets, and then rotated such that the mean vector was aligned with $v = (0,\ldots,0,1)$.
The rotation then ensured, that the mean direction gets mapped to the origin of the unembedded space.
The unembedding then used a constant $s$ value of $1$.
We did not tune any of the parameters of ScaNN and Annoy, but used constant parameters for all runs, since we were only interested in the relative improvement due to (un)embedding the data.
The recall of ScaNN and Annoy improves for most of the datasets significantly, where we (un)embedded the data from a less appropriate into a more appropriate space.
Exceptions were the Last.fm, Fashion MNIST and SIFT datasets.
Unembedding non-spherical data or embedding spherical data can at best improve the recall slightly.
Thus, the recall does not always improve and identifying cases in which spherical (un)embedding is advantageous remains a matter of future research, but in the cases where an improvement was expected and observed, it was quite significant.

The spherical embedding could further be applied in neural networks as an activation function, to virtually replace the dot products with soft distance measures to learned centers.
A visual example is given in \reffig{fig:neural_nets}, where two (almost) identical networks are trained on the moons dataset with and without spherical embedding as the first activation.
Not only does the spherical embedding network converge faster, it produces \enquote{softer} decision boundaries, since the decisions in original space are based on distances to virtual centers rather than hyperplanes.
To allow for a shift of the center of inversion, a linear layer was added before the spherical embedding.
For the spherical embedding network, this layer was initialized with the identity matrix and zero bias, yet randomly for the baseline model, that used a ReLU activation instead of the embedding.

These examples are some inspirations as to where the results of this paper can be leveraged.
Aside from potentially improving the quality of approaches tailored to either spherical or non-spherical data in various machine learning tasks, the spherical embedding adds almost no computational cost.
Translating the caps to balls or vice-versa after training, allows to skip the (un)embedding process after training, although existing libraries need to include that shortcut programmatically.

\section{Conclusion}
We have proposed spherical (un)embedding of data using arbitrarily-dimensional spherical inversion and introduced an explicit recommendation for the directional vector $v$ and a selection scheme for the scale parameter $s$ of these embeddings.
Using these (un)embedding functions allows to apply algorithms to other domains or to translate data into a more appropriate domain, achieving improved execution time or quality on outlier detection, classification, and spatial indexing.
We have further given explicit formulae to translate between hyperspherical caps in the spherically embedded space and balls in the original space.
This allows to reduce the computational overhead of the (un)embedding process after training by avoiding the (un)embedding altogether and instead evaluating distances or dot products in the task domain.
By providing closed forms for inner products and distances in the respective other space using only functions of the original space, it is possible to apply these results to kernel-based approaches as well.
This article proves the relevant formulae for future research on methods that exploit the Cap-Ball-Duality in machine learning research.

\bibliography{main}
\bibliographystyle{tmlr}

\clearpage
\appendix
\section{Appendix}

\begin{proof}[Proof of \reftheorem{theorem:cap_ball_duality}]
	Let $S(x)$ and $S^{-1}(x)$ be shorthand for $\simplified{S}(x, v, s)$ and $\simplified{S}^{-1}(x, v, s)$, respectively.
	Let $p$ and $b$ be the directional unit-length vector and bias of an arbitrary hyperspherical cap satisfying $b+p_{d+1} > 0$.
	We choose $c$ and $r$ as described above.
	Since $r$ decreases for increasing $b$, a smaller cap results in a smaller ball and vice versa.
	By continuity of $S,~S^{-1},$ and our definition of $r$, it suffices to show, that the boundary of the cap and the ball are images of each other under $S$ and $S^{-1}$.
	We, therefore, assume an arbitrary $x \in \nsphere{d}$ with $\mdot{x}{p} = b$ and intend to show that $\norm{S^{-1}(x) - c} = r$.
	We will make use of the fact, that for such $x$ holds $\Vert x + v \Vert = 2(1+x_{d+1})$ since
	\begin{eqnarray*}
		\Vert x + v \Vert^2 &=& \mdot{x+v}{x+v}
		= \mdot{x}{x} + \mdot{v}{v} + 2\mdot{x}{v}
		= 2(1+x_{d+1}).
	\end{eqnarray*}
	Let $\hat{c} := S(c) = \frac{p - \alpha v}{\norm{p - \alpha v}}$, then
	\allowdisplaybreaks
	\begin{align*}
		&\norm{S^{-1}(x) - S^{-1}(\hat{c})}\\
		&= \norm{2s \left(\frac{x + v}{\sqnorm{x + v}}\right)_{1,\ldots,d} - 2s \left(\frac{\hat{c} + v}{\sqnorm{\hat{c} + v}}\right)_{1,\ldots,d}}\\
		&= \norm{2s \left(\frac{x}{\sqnorm{x + v}}\right)_{1,\ldots,d} - 2s \left(\frac{\hat{c}}{\sqnorm{\hat{c} + v}}\right)_{1,\ldots,d}}\\
		&= s \norm{\left(\frac{x}{1+x_{d+1}} - \frac{\hat{c}}{1+\hat{c}_{d+1}}\right)_{1,\ldots,d}}\\
		&= s \sqrt{\sum_{i=1}^d \frac{x_i^2}{(1+x_{d+1})^2} + \frac{\hat{c}_i^2}{(1+\hat{c}_{d+1})^2} - 2\frac{x_i\hat{c}_i}{(1+x_{d+1})(1+\hat{c}_{d+1})}}\\
		&= s \sqrt{
			\frac{\sqnorm{x_{1,\ldots,d}}}{(1+x_{d+1})^2}
			+ \frac{\sqnorm{\hat{c}_{1,\ldots,d}}}{(1+\hat{c}_{d+1})^2}
			- 2\frac{\mdot{x_{1,\ldots,d}}{\hat{c}_{1,\ldots,d}}}{(1+x_{d+1})(1+\hat{c}_{d+1})}
		}\\
		&= s \sqrt{
			\frac{1-x_{d+1}^2}{(1+x_{d+1})^2}
			+ \frac{1-\hat{c}_{d+1}^2}{(1+\hat{c}_{d+1})^2}
			- 2\frac{\mdot{x_{1,\ldots,d}}{\hat{c}_{1,\ldots,d}}}{(1+x_{d+1})(1+\hat{c}_{d+1})}
		}\\
		&= s \sqrt{
			\frac{1-x_{d+1}}{1+x_{d+1}}
			+ \frac{1-\hat{c}_{d+1}}{1+\hat{c}_{d+1}}
			- 2\frac{\mdot{x_{1,\ldots,d}}{\hat{c}_{1,\ldots,d}}}{(1+x_{d+1})(1+\hat{c}_{d+1})}
		}\\
		&= s \sqrt{
			\frac{
				(1-x_{d+1})(1+\hat{c}_{d+1})
				+ (1-\hat{c}_{d+1})(1+x_{d+1})
				- 2\mdot{x_{1,\ldots,d}}{\hat{c}_{1,\ldots,d}}
			}{(1+x_{d+1})(1+\hat{c}_{d+1})}
		}\\
		&= s \sqrt{
			\frac{
				2 - 2x_{d+1}\hat{c}_{d+1}
				- 2\mdot{x_{1,\ldots,d}}{\hat{c}_{1,\ldots,d}}
			}{(1+x_{d+1})(1+\hat{c}_{d+1})}
		}\\
		&= s \sqrt{
			\frac{2 - 2\mdot{x}{\hat{c}}}{(1+x_{d+1})(1+\hat{c}_{d+1})}
		}\\
		&= \sqrt{2}s \sqrt{
			\frac{1 - \mdot{x}{\hat{c}}}{(1+x_{d+1})(1+\hat{c}_{d+1})}
		}\\
		\intertext{At this point, we have reduced the formula as far as possible in terms of $\hat{c}$ and now substitute $\hat{c}$ with the definition given in the Theorem.}
		&= \sqrt{2}s \sqrt{
			\frac{\norm{p - \alpha v} - \mdot{x}{p - \alpha v}}{(1+x_{d+1})\left(\norm{p - \alpha v}+\left(p - \alpha v\right)_{d+1}\right)}
		}\\
		&= \sqrt{2}s \sqrt{\frac{
			\norm{p - \alpha v} - b + \alpha x_{d+1}
		}{
			(1+x_{d+1})\left(\norm{p - \alpha v}+p_{d+1} - \alpha\right)
		}}\\
		&= \sqrt{2}s \sqrt{\frac{
			\sqrt{1+\alpha^2-2\alpha p_{d+1}} - b + \alpha x_{d+1}
		}{
			(1+x_{d+1})\left(\sqrt{1+\alpha^2-2\alpha p_{d+1}}+p_{d+1} - \alpha\right)
		}}\\
		&= \sqrt{2}s \sqrt{\frac{
			\sqrt{1+\left(\frac{1-b^2}{2(b+p_{d+1})}\right)^2-\frac{1-b^2}{b+p_{d+1}} p_{d+1}} - b + \frac{1-b^2}{2(b+p_{d+1})} x_{d+1}
		}{
			(1+x_{d+1})\left(\sqrt{1+\left(\frac{1-b^2}{2(b+p_{d+1})}\right)^2-\frac{1-b^2}{b+p_{d+1}} p_{d+1}}+p_{d+1} - \frac{1-b^2}{2(b+p_{d+1})}\right)
		}}\\
		&= \sqrt{2}s \sqrt{\frac{
			\sqrt{4(b+p_{d+1})^2+(1-b^2)^2-4(b+p_{d+1})(1-b^2) p_{d+1}} - 2b(b+p_{d+1}) + (1-b^2) x_{d+1}
		}{
			(1+x_{d+1})\left(\sqrt{4(b+p_{d+1})^2+(1-b^2)^2-4(b+p_{d+1})(1-b^2) p_{d+1}} + 2p_{d+1}(b+p_{d+1}) - (1-b^2)\right)
		}}\\
		&= \sqrt{2}s \sqrt{\frac{
			\sqrt{
				b^4
				+ 4b^3p_{d+1}
				+ 4b^2p_{d+1}^2
				+ 2b^2
				+ 4bp_{d+1}
				+ 1
			}
			- 2b^2
			- 2bp_{d+1}
			+ (1-b^2) x_{d+1}
		}{
			(1+x_{d+1})\left(
				\sqrt{
					b^4
					+ 4b^3p_{d+1}
					+ 4b^2p_{d+1}^2
					+ 2b^2
					+ 4bp_{d+1}
					+ 1
				}
				+ 2bp_{d+1}
				+ 2p_{d+1}^2
				- (1-b^2)
			\right)
		}}\\
		&= \sqrt{2}s \sqrt{\frac{
			(b^2+2bp_{d+1}+1)
			- 2b^2
			- 2bp_{d+1}
			+ (1-b^2) x_{d+1}
		}{
			(1+x_{d+1})\left(
				(b^2+2bp_{d+1}+1)
				+ 2bp_{d+1}
				+ 2p_{d+1}^2
				- (1-b^2)
			\right)
		}}\\
		&= \sqrt{2}s \sqrt{\frac{
			(1-b^2) + (1-b^2) x_{d+1}
		}{
			(1+x_{d+1})2\left(b + p_{d+1}\right)^2
		}}\\
		&= s \sqrt{\frac{1-b^2}{\left(b + p_{d+1}\right)^2}}\\
		&= r
	\end{align*}
	Thus all points from the boundary of the cap have a distance of $r$ to the given center $c$, forming a sphere of radius $r$.
\end{proof}

\begin{proof}[Proof of \reftheorem{theorem:cap_ball_duality2}]
	Inserting the solution for $\alpha$ in \reftheorem{theorem:cap_ball_duality} into the equation for $r$ and solving for $b$ gives
	\begin{eqnarray*}
		b &=& \frac{s \sqrt{s^2+(1-p_{d+1}^2)r^2} - p_{d+1}r^2}{r^2+s^2}
	\end{eqnarray*}
	Solving the equation for $r$ for $\alpha$ and substituting $b$ with the solution for $b$ gives
	\begin{eqnarray*}
		\alpha &=& \frac{r^2\left(p_{d+1}s + \sqrt{s^2 + (1-p_{d+1}^2)r^2}\right)}{2s(r^2+s^2)}
	\end{eqnarray*}
	Inserting that value for $\alpha$ in Equation (\ref{eq:ball_center_simplified}) and solving for $p_{1,\ldots,d}$ gives
	\begin{eqnarray*}
		p_{1,\ldots,d} &=& \frac{
			c\left(\begin{array}{l}
				p_{d+1}s(r^2+2s^2)
				- r^2\sqrt{\sqnorm{p_{1,\ldots,d}}r^2+s^2}
				\\+ \sqrt{\begin{array}{l}
					4\sqnorm{p_{1,\ldots,d}}s^2(r^2+s^2)^2
					\\+ \left(
						p_{d+1}s(r^2+2s^2)
						-r^2\sqrt{\sqnorm{p_{1,\ldots,d}}r^2+s^2}
					\right)^2
				\end{array}}
			\end{array}\right)
		}{2s^2(r^2+s^2)}
	\end{eqnarray*}
	which we can solve for $p_{d+1}$ by multiplying both vectors with $p_{1,\ldots,d}$ and solving
	\begin{eqnarray*}
		p_{d+1} = \frac{
			\sqnorm{p_{1,\ldots,d}}(r^2+s^2)(s^2-\sqnorm{c})+
			\mdot{c}{p_{1,\ldots,d}}r^2\sqrt{\sqnorm{p_{1,\ldots,d}}r^2+s^2}
		}{s(r^2+2s^2)\mdot{c}{p_{1,\ldots,d}}}
	\end{eqnarray*}
	Using the definition, that $p_{1,\ldots,d} = \beta c$, we can give a closed form for $p_{d+1}$ as
	\begin{eqnarray*}
		p_{d+1} = \frac{
			\beta(r^2+s^2)(s^2-\sqnorm{c}) + r^2\sqrt{s^2+\sqnorm{c}\beta^2r^2}
		}{s(r^2+2s^2)}
	\end{eqnarray*}
	Equating this solution with $\sqrt{1-\beta^2\sqnorm{c}}$, which is derived from the unit length of $p$, we obtain the final solution for $\beta$ that is given in the Theorem.
	The solution is unique whenever it exists, since the center of a hyperball and direction vector of a cap are unique.
	The solution always exists, iff all terms are well-defined, more specifically iff the term under the root in the solution for $\beta$ is always positive.
	The term becomes smallest for $s \to 0$ and in the limit $s \to 0$, the term becomes $(c^2-r^2)^2$ which obviously is positive.
	Since we demand $s > 0$, the term is always positive and the solution for $\beta$ is always well-defined.
\end{proof}

\begin{proof}[Proof of \refcorollary{corollary:kernel_functions}]
	In the proofs we will make use of the simplified definitions $\simplified{S}$ and $\simplified{S}^{-1}$ for the special case of $v=(0,\ldots,0,1)$.
	First we will derive the equations when moving from Euclidean to spherical space and then we will derive the equations for the inverse direction.
	\allowdisplaybreaks
	\begin{description}
		\item[Product, Euclidean to spherical]
		\begin{eqnarray*}
			&&\mdot{\hat{x}}{\hat{y}}\\
			&=& \mdot{2s\frac{\veccat{x}{s}}{\sqnorm{x}+s^2}-v}{2s\frac{\veccat{y}{s}}{\sqnorm{y}+s^2}-v}\\
			&=& \mdot{2s\frac{\veccat{x}{s}}{\sqnorm{x}+s^2}}{2s\frac{\veccat{y}{s}}{\sqnorm{y}+s^2}}
			\\&&
			- \mdot{2s\frac{\veccat{x}{s}}{\sqnorm{x}+s^2}}{v}
			- \mdot{2s\frac{\veccat{y}{s}}{\sqnorm{y}+s^2}}{v}
			+ \mdot{v}{v}\\
			&=& \mdot{2s\frac{x}{\sqnorm{x}+s^2}}{2s\frac{y}{\sqnorm{y}+s^2}}
			+ \left(\frac{2s^2}{\sqnorm{x}+s^2}\right)\left(\frac{2s^2}{\sqnorm{y}+s^2}\right)
			\\&&
			- \frac{2s^2}{\sqnorm{x}+s^2}
			- \frac{2s^2}{\sqnorm{y}+s^2}
			+ 1\\
			&=& \frac{4s^2\mdot{x}{y}}{(\mdot{x}{x}+s^2)(\mdot{y}{y}+s^2)}
			+\left(\frac{2s^2}{\mdot{x}{x}+s^2} - 1\right)
			\left(\frac{2s^2}{\mdot{y}{y}+s^2} - 1\right)\\
			&=& \frac{
				4s^2\mdot{x}{y}
				+ \left(2s^2 - (\mdot{x}{x}+s^2)\right)
				\left(2s^2 - (\mdot{y}{y}+s^2)\right)
			}{(\mdot{x}{x}+s^2)(\mdot{y}{y}+s^2)}\\
			&=& \frac{
				4s^2\mdot{x}{y}
				+ 4s^4
				- 2s^2(\mdot{x}{x}+s^2)
				- 2s^2(\mdot{y}{y}+s^2)
			}{(\mdot{x}{x}+s^2)(\mdot{y}{y}+s^2)}
			+ 1\\
			&=& 1-2s^2\frac{
				\mdot{x}{x} -2\mdot{x}{y} + \mdot{y}{y}
			}{(\mdot{x}{x}+s^2)(\mdot{y}{y}+s^2)}
		\end{eqnarray*}
		\item[Distance, Euclidean to spherical]
		\begin{eqnarray*}
			&&\sqnorm{\hat{x}-\hat{y}}
			\\&=& 2 - 2 \mdot{\hat{x}}{\hat{y}}
			\\&=& 2 - 2\left(1-2s^2\frac{
				\mdot{x}{x} -2\mdot{x}{y} + \mdot{y}{y}
			}{(\mdot{x}{x}+s^2)(\mdot{y}{y}+s^2)}\right)
			\\&=& \frac{4s^2}{(\sqnorm{x}+s^2)(\sqnorm{y}+s^2)} \sqnorm{x-y}
		\end{eqnarray*}
		\item[Product, spherical to Euclidean]
		\begin{eqnarray*}
			\mdot{x}{y}
			&=& \mdot{s\frac{\hat{x}_{1,\ldots,d}}{1+\hat{x}_{d+1}}}{s\frac{\hat{y}_{1,\ldots,d}}{1+\hat{y}_{d+1}}}
			\\&=& s^2\frac{\mdot{\hat{x}}{\hat{y}} - \hat{x}_{d+1}\hat{y}_{d+1}}{(1+\hat{x}_{d+1})(1+\hat{y}_{d+1})}
		\end{eqnarray*}
		\item[Distance, spherical to Euclidean]
		\begin{eqnarray*}
			&&\sqnorm{x-y}
			\\&=& \sqnorm{x} + \sqnorm{y} - 2 \mdot{x}{y}
			\\&=& s^2\frac{\mdot{\hat{x}}{\hat{x}} - \hat{x}_{d+1}^2}{(1+\hat{x}_{d+1})^2} + s^2\frac{\mdot{\hat{y}}{\hat{y}} - \hat{y}_{d+1}^2}{(1+\hat{y}_{d+1})^2} - 2 \left(
				s^2\frac{\mdot{\hat{x}}{\hat{y}} - \hat{x}_{d+1}\hat{y}_{d+1}}{(1+\hat{x}_{d+1})(1+\hat{y}_{d+1})}
			\right)
			\\&=& s^2\frac{1 - \hat{x}_{d+1}}{1+\hat{x}_{d+1}} + s^2\frac{1 - \hat{y}_{d+1}}{1+\hat{y}_{d+1}} - 2 \left(
				s^2\frac{\mdot{\hat{x}}{\hat{y}} - \hat{x}_{d+1}\hat{y}_{d+1}}{(1+\hat{x}_{d+1})(1+\hat{y}_{d+1})}
			\right)
			\\&=& s^2\frac{\left(\begin{array}{c}
				(1-\hat{x}_{d+1})(1+\hat{y}_{d+1})
				+ (1+\hat{x}_{d+1})(1-\hat{y}_{d+1})\\
				- 2 \left(\mdot{\hat{x}}{\hat{y}} - \hat{x}_{d+1}\hat{y}_{d+1}\right)
			\end{array}\right)}{(1+\hat{x}_{d+1})(1+\hat{y}_{d+1})}
			\\&=& s^2\frac{
				2 - 2 \mdot{\hat{x}}{\hat{y}}
			}{(1+\hat{x}_{d+1})(1+\hat{y}_{d+1})}
			\\&=& \frac{s^2}{(1+\hat{x}_{d+1})(1+\hat{y}_{d+1})}\sqnorm{\hat{x}-\hat{y}}
		\end{eqnarray*}
	\end{description}
\end{proof}

\end{document}